\documentclass[nohyperref]{article}

\usepackage{microtype}
\usepackage{graphicx}
\usepackage{booktabs} 
\usepackage[font=footnotesize,labelformat=simple]{subcaption}

\usepackage{hyperref}

\usepackage[accepted]{icml2022}

\usepackage{amsmath}
\usepackage{amssymb}
\usepackage{mathtools}
\usepackage{amsthm}
\usepackage{multirow}
\usepackage{tabularx}
\usepackage{color}

\usepackage[capitalize,noabbrev]{cleveref}

\theoremstyle{plain}
\newtheorem{theorem}{Theorem}[section]

\newtheorem{lemma}[theorem]{Lemma}

\theoremstyle{definition}

\theoremstyle{remark}

\usepackage[textsize=tiny]{todonotes}

\icmltitlerunning{Robust Task Representations for Offline Meta-Reinforcement Learning via Contrastive Learning}

\begin{document}

\twocolumn[
\icmltitle{Robust Task Representations for Offline Meta-Reinforcement Learning \\ via Contrastive Learning}



\icmlsetsymbol{equal}{*}

\begin{icmlauthorlist}
\icmlauthor{Haoqi Yuan}{yyy}
\icmlauthor{Zongqing Lu}{yyy}
\end{icmlauthorlist}

\icmlaffiliation{yyy}{School of Computer Science, Peking University}

\icmlcorrespondingauthor{Zongqing Lu}{zongqing.lu@pku.edu.cn}

\icmlkeywords{Machine Learning, ICML}

\vskip 0.3in
]



\printAffiliationsAndNotice{} 

\begin{abstract}
We study offline meta-reinforcement learning, a practical reinforcement learning paradigm that learns from offline data to adapt to new tasks. The distribution of offline data is determined jointly by the behavior policy and the task. Existing offline meta-reinforcement learning algorithms cannot distinguish these factors, making task representations unstable to the change of behavior policies. To address this problem, we propose a contrastive learning framework for task representations that are \textit{robust} to the distribution mismatch of behavior policies in training and test. We design a bi-level encoder structure, use mutual information maximization to formalize task representation learning, derive a contrastive learning objective, and introduce several approaches to approximate the true distribution of negative pairs. Experiments on a variety of offline meta-reinforcement learning benchmarks demonstrate the advantages of our method over prior methods, especially on the generalization to out-of-distribution behavior policies. 
\end{abstract}

\section{Introduction}
\label{introduction}

Deep reinforcement learning (RL) has achieved great successes in playing video games \cite{atari-drl,fps-drl}, robotics \cite{robot-drl}, recommendation systems \cite{recommend-drl} and multi-agent systems \cite{cooperative-marl}. However, deep RL still faces two challenging problems: data efficiency and generalization. To play Atari games, model-free RL takes millions of steps of environment interactions, while model-based RL takes 100k environment steps \cite{modelbased-atari}. The tremendous interactions along with safety issues prevent its applications in many real-world scenarios. Offline reinforcement learning \cite{offline-review} tackles the sample efficiency problem by learning from pre-collected offline datasets, without any online interaction. On the other hand, as deep RL is supposed to be tested in the same environment to the training environment, when the transition dynamics or the reward function changes, the performance degenerates. To tackle the generalization problem, meta-reinforcement learning \cite{maml,rl2} introduces learning over task distribution to adapt to new tasks.

Offline Meta-Reinforcement Learning (OMRL), an understudied problem, lies in the intersection of offline RL and meta-RL. Usually, the offline dataset is collected from multiple tasks by different behavior policies. The training agent aims at learning a meta-policy, which is able to efficiently adapt to unseen tasks. Recent studies \cite{mbml,borel,improved-context,focal} extend context-based meta-RL to OMRL. They propose to use a context encoder to learn task representations from the collected trajectories, then use the latent codes as the policy's input. 

However, context-based methods are vulnerable to the distribution mismatch of behavior policies in training and test phases. The distribution of collected trajectories depends both on the behavior policy and the task. When behavior policies are highly correlated with tasks in the training dataset, the context encoder is likely to memorize the feature of behavior policies. Thus, in the test phase, the context encoder produces biased task inference due to the change of the behavior policy. Although \citet{focal, improved-context} employed contrastive learning to improve task representations, their learning objectives are simply based on discriminating trajectories from different tasks, and thus cannot eliminate the influence of behavior policies.

To overcome this limitation, we propose a novel framework of \textbf{CO}ntrastive \textbf{R}obust task \textbf{R}epresentation learning for \textbf{O}MRL (\textbf{CORRO}). We design a bi-level structured task encoder, where the first level extracts task representations from one-step transition tuples instead of trajectories and the second level aggregates the representations. We formalize the learning objective as mutual information maximization between the representation and task, to maximally eliminate the influence of behavior policies from task representations. We introduce a contrastive learning method to optimize for InfoNCE, a mutual information lower bound. To approximate the negative pairs' distribution, we introduce two approaches for negative pairs generation, including generative modeling and reward randomization. Experiments in Point-Robot environment and multi-task MuJoCo benchmarks demonstrate the substantial performance gain of CORRO over prior context-based OMRL methods, especially when the behavior policy for adaptation is out-of-distribution.

\textbf{Our main contributions, among others, are: }
\begin{itemize}
    \vspace{-0.30cm}
	\setlength\itemsep{2pt}
    \item We propose a framework for learning robust task representations with \textit{fully} offline datasets, which can distinguish tasks from the distributions of transitions jointly determined by the behavior policy and task. 
    \item We derive a contrastive learning objective to extract shared features in the transitions of the same task while capturing the essential variance of reward functions and transition dynamics across tasks. 
    \item We empirically show on a variety of benchmarks that our method much better generalizes to out-of-distribution behavior policies than prior methods, even better than \textit{supervised task learning} that assumes the ground-truth task descriptions.  
\end{itemize}

\section{Related Work}

\textbf{Offline Reinforcement Learning} allows policy learning from data collected by arbitrary policies, increasing the sample efficiency of RL. In off-policy RL \cite{atari-drl, ddqn, sac, ddpg, awr}, the policy reuses the data collected by its past versions, and learns based on Q-learning or importance sampling. Batch RL studies learning from fully offline data. Recent works \cite{batch-rl, bcq, offline-review, behavior-regularize, benchmark-batch-rl, bootstrap-offline-q} propose methods to overcome distributional shift and value overestimation issues in batch RL. Our study follows the batch RL setting, but focuses on learning task representations and meta-policy from offline multi-task data.

\textbf{Meta-Reinforcement Learning} learns to quickly adapt to new tasks via training on a task distribution. Context-based methods \cite{rl2,varibad, pearl, meta-q, concurrent-meta-rl} formalize meta-RL as POMDP, regard tasks as unobservable parts of states, and encode task information from history trajectories. Optimization-based methods \cite{maml, promp, emaml, continuous-adapt, dice, evolved-pg} formalize task adaptation as performing policy gradients over few-shot samples and learn an optimal policy initialization. Our study is based on the framework of context-based meta-RL.

\textbf{Offline Meta-Reinforcement Learning} studies learning to learn from offline data. Because there is a distributional mismatch between offline data and online explored data during test, learning robust task representations is an important issue. Recent works \cite{focal, improved-context} apply contrastive learning over trajectories for compact representations, but ignore the influence of behavior policy mismatch. To fix the distributional mismatch, \citet{mbml, borel} assume \textit{known} reward functions of different tasks, and \citet{omrl-online} require additional online exploration. Unlike them, we consider learning robust task representations with fully offline data.

\textbf{Contrastive Learning} \cite{moco, byol,alignment-and-uniform} is a popular method for self-supervised representation learning. It constructs positive or negative pairs as noisy versions of samples with the same or different semantics. Via distinguishing the positive pair among a large batch of pairs, it extracts meaningful features. Contrastive learning has been widely applied in computer vision \cite{simclr, moco, generalized-data-transform}, multi-modal learning \cite{tcn, contrastive-multiview, p4contrast} and image-based reinforcement learning \cite{contrastive-atari, decouple-state-representation, curl}. In our study, we apply contrastive learning in offline task representation learning, for robust OMRL.

\section{Preliminaries}
\subsection{Problem Formulation}

A task for reinforcement learning is formalized as a fully observable Markov Decision Process (MDP). MDP is modeled as a tuple $M=(\mathcal{S}, \mathcal{A}, T, \rho, R, \gamma)$, where $\mathcal{S}$ is the state space, $\mathcal{A}$ is the action space, $T(s'|s,a)$ is the transition dynamics of the environment, $\rho(s)$ is the initial state distribution, $R(s,a)$ is the reward function, $\gamma\in[0,1)$ is the factor discounting the future reward. The policy of the agent is a distribution $\pi(a|s)$ over actions. Starting from the initial state, for each time step, the agent performs an action sampled from $\pi$, then the environment updates the state with $T$ and returns a reward with $R$. We denote the marginal state distribution at time $t$ as $\mu_\pi^t(s)$.  The objective of the agent is to maximize the expected cumulative rewards $\max_\pi J_M(\pi)=\mathbb{E}_{s_t\sim \mu_\pi^t, a_t\sim \pi}{[\sum_{t=0}^{\infty}{\gamma^t R(s_t,a_t)}]}$. To evaluate and optimize the policy, we define V-function and Q-function as follows:
\vspace{-0.1cm}
\begin{gather}
    V_\pi(s)=\sum_{t=0}^\infty{\gamma^t \mathbb{E}_{s_t\sim \mu_\pi^t, a_t\sim \pi} \left[R(s_t,a_t)\right]} \\
    Q_\pi(s,a) = R(s,a)+\gamma\mathbb{E}_{s'\sim T(s'|s,a)}{[V_\pi(s')]}.
\end{gather}
Q-learning solves for the optimal policy by iterating the Bellman optimality operator $\mathcal{B}$ over Q-function:
\begin{equation}
    \mathcal{B}\hat{Q}(s,a)=R(s,a)+\gamma\mathbb{E}_{s'\sim T(s'|s,a)}{\left[\max_{a'}{\hat{Q}(s',a')}\right]}.
\end{equation}
In offline meta-reinforcement learning (OMRL), we assume that the task follows a distribution $M_i=(\mathcal{S}, \mathcal{A}, T_i, \rho, R_i, \gamma) \sim P(M) $. Tasks share the same state space and action space, but varies in reward functions and transition dynamics. Thus, we can also denote the task distribution as $P(R,T)$. 
Given $N$ training tasks $\{M_i\}_{i=1}^{N}$, for each task $i$, an offline dataset $X_i=\{(s_{i,j},a_{i,j},r_{i,j},s'_{i,j})\}_{j=1}^K$ is collected by arbitrary behavior policy $\pi_\beta^i$. The learning algorithms can only access the offline datasets to train a meta-policy $\pi_{\operatorname{meta}}$, without any environmental interactions. At test time, given an unseen task $M \sim P(M)$, an arbitrary exploration (behavior) policy collects a context $c=\{(s_j,a_j,r_j,s'_j)\}_{j=1}^k$, then the learned agent performs task adaptation conditioned on $c$ to get a task-specific policy $\pi_M$ and evaluate in the environment. The objective for OMRL is to learn a meta policy maximize the expected return over the test tasks:
\vspace{-0.1cm}
\begin{equation}
    J(\pi_{\operatorname{meta}})=\mathbb{E}_{M\sim P(M)}{\left[J_M(\pi_M)\right].}
\end{equation}

\subsection{Context-Based OMRL and FOCAL}

Context-based learning regards OMRL as solving partially observable MDPs. Considering the task $M$ as the unobservable part of the state, the agent gathers task information and makes decisions upon the history trajectory: $a_t\sim \pi(a|\tau_{0:t-1}, s_t)$, where $\tau_{0:t-1}=(s_0,a_0,r_0,\cdots,s_{t-1},a_{t-1},r_{t-1})$. 

A major assumption in context-based learning is that different tasks share some common structures and the variation over tasks can be described by a compact representation. For example, in a robotic object manipulation scenario, tasks can be described as object mass, friction, and target position, while rules of physical interaction are shared across tasks. Context-based OMRL uses a task encoder to learn a latent space $z_t=E(\tau_{0:t-1})$, to represent task information \cite{pearl,focal} or task uncertainty \cite{borel}. The policy $\pi(a|s,z)$ is conditioned on the latent task representation. Given a new task, the exploration policy collects a few trajectories (context), and the task encoder adapts the policy by producing $z$. In principle, the task encoder and the policy can be trained on offline datasets with off-policy or batch RL methods. \textit{In our setting, we assume that the exploration policy is arbitrary, and the context is a single trajectory.} During training, the context is sampled from the offline dataset.

FOCAL \cite{focal} improves task representation learning of the encoder via distance metric learning. The encoder minimizes the distance of trajectories from the same task and pushes away trajectories from different tasks in the latent space. The loss function is
\begin{equation}
\begin{aligned}
    \mathcal{L}_{dml}=& \boldsymbol{1}\{y_i=y_j\} \| q_i-q_j\|_2^2 + \\
    &\boldsymbol{1}\{y_i\neq y_j\} \beta\cdot \frac{1}{\| q_i-q_j\|_2^n + \epsilon} 
\end{aligned}
\end{equation}
where $q_i=E(\tau_i)$ is the representation of a trajectory, $y_i$ is the identifier of the offline dataset. The policy is separately trained, conditioned on the learned encoder, with batch Q-learning methods.

\subsection{Task Representation Problems in OMRL}
Since the offline datasets are collected by different behavior policies, to answer which dataset a trajectory belongs to, one may infer based on the feature of behavior policy rather than rewards and state transitions. 
As an example: In a 2D goal reaching environment, tasks differ in goal positions. In each training task, the behavior policy is to go towards the goal position. Training on such datasets, the task encoder in FOCAL can simply ignore the rewards and distinguish the tasks based on the state-action distribution in the context. Thus, it will make mistakes when the context exploration policy changes.

Similar problem is also mentioned in \citet{borel, mbml}. To eliminate the effects of behavior policies in task learning, \citet{borel} augment the datasets by collecting trajectories in different tasks with the same policy, and relabeling the transition tuples with different tasks' reward functions. However, in our fully offline setting, collecting additional data and accessing the reward functions are not allowed. \citet{mbml} use the learned reward functions in different tasks to relabel the trajectories and apply metric learning over trajectories. This approach does not support the settings where tasks differ in transition dynamics, because we cannot perform transition relabeling of $s'$ over the trajectory to mimic the trajectory in other tasks. Also, the reward function learned with limited single-task dataset can be inaccurate to relabel the unseen state-action pairs.

\section{Method}
To address the offline task representation problem, we propose CORRO, a novel contrastive learning framework for robust task representations with fully offline datasets, which decreases the influence of behavior policies on task representations while supporting tasks that differ in reward function and transition dynamics.

\begin{figure}[!t]
\centering
\includegraphics[scale=0.4, trim={0cm, 7.5cm, 13.4cm, 0cm}, clip]{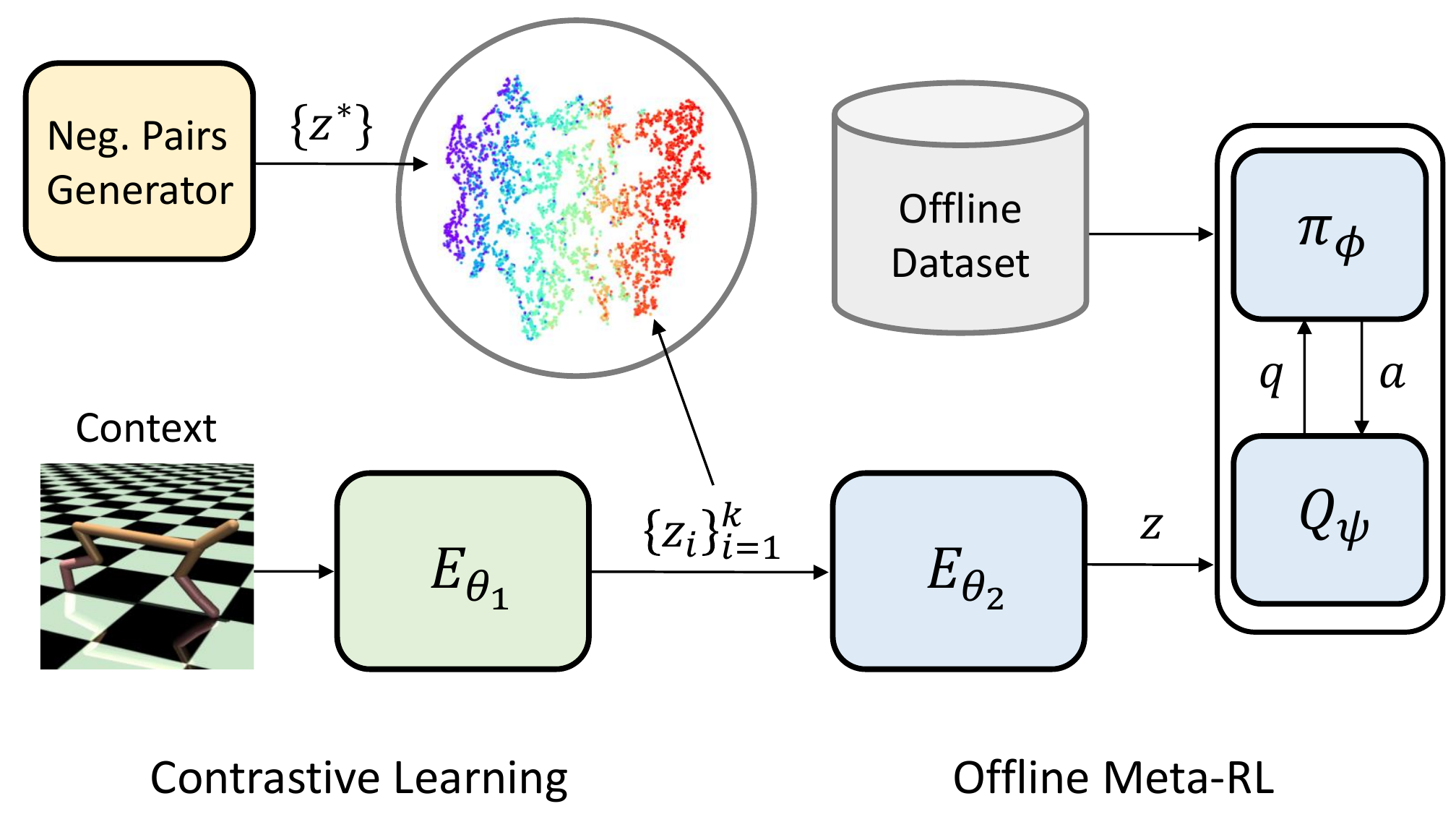}
\caption{CORRO framework. The transition encoder $E_{\theta_1}$ extracts the latent representations of each transition tuple in the context. Contrastive learning is applied to robustify the task encoder. The aggregator $E_{\theta_2}$ gathers all the latent codes, to condition the actor and critic.}
\label{fig:framework}
\end{figure}

\subsection{Bi-Level Task Encoder}
We first design a bi-level task encoder where representation learning focuses on transition tuples rather than trajectories. Concretely, our task encoder consists of a transition encoder $E_{\theta_1}$ and an aggregator $E_{\theta_2}$, parameterized collectively as $\theta$. Given a context $c=\{(s_i,a_i,r_i,s_i')\}_{i=1}^k$, the transition encoder extracts latent representations for all the transition tuples $z_i = E_{\theta_1}(s_i,a_i,r_i,s_i')$, then the aggregator gather all the latent codes into a task representation $z=E_{\theta_2}(\{z_i\}_{i=1}^k)$. The Q-function $Q_\psi(s,a,z)$ and the policy $\pi_\phi(a|s,z)$ are conditioned on $z$, parameterized with $\phi, \psi$.

Compared to trajectories, transition tuples leak less behavior policy information and are suitable for both transition relabeling and reward relabeling. Thus, we design robust task representation learning methods to train only the transition encoder in the following sections. The architecture of the transition encoder is simple MLP. The aggregator is trained with Q-function and policy using offline RL algorithms. Intuitively, it should draw attention to the transition tuples with rich task information. Inspired by self-attention \cite{attention}, we introduce an aggregator which computes the weighted sum of the input vectors:
\begin{equation}
    z = \sum_{j=1}^k{\operatorname{softmax}\left(\{\operatorname{MLP}(z_i)\}_{i=1}^k\right)_j \cdot z_j}.
\end{equation}
Figure \ref{fig:framework} gives an overview of our framework.

\subsection{Contrastive Task Representation Learning}
\label{sec:contrastive}


An ideally robust task representation should be dependent on tasks and invariant across behavior policies. We introduce mutual information to measure the mutual dependency between the task representation and the task. As, intuitively, mutual information measures the uncertainty reduction of one random variable when the other one is observed, we propose to maximize the mutual information between the task representation and the task itself, so that the task encoder learns to maximally reduce task uncertainty while minimally preserving task-irrelevant information.

Mathematically, we formalize the transition encoder $E_{\theta_1}$ as a probabilistic encoder $z\sim P(z|x)$, where $x=(s,a,r,s')$ denotes the transition tuple. Task $M$ follows the task distribution $P(M)$ and the distribution of $x$ is determined jointly by $M$ and the behavior policy. The learning objective for the transition encoder is:
\begin{equation}
\label{eq:mutual}
    \max I(z;M)=\mathbb{E}_{z,M}{\left[\log{\frac{p(M|z)}{p(M)}}\right]}.
\end{equation}

Optimizing mutual information is intractable in practice. Inspired by noise contrastive estimation (InfoNCE) in the literature of contrastive learning \cite{cpc}, we derive a lower bound of Eq. \eqref{eq:mutual} and have the following theorem.

\begin{theorem}
\label{thm:nce}
Let $\mathcal{M}$ be a set of tasks following the task distribution, $|\mathcal{M}|=N$. $M\in\mathcal{M}$ is the first task with reward and transition $R,T$. Let $x=(s,a,r,s'), z\sim P(z|x), h(x,z)=\frac{P(z|x)}{P(z)}$, where $(s,a)$ follows an arbitrary distribution, $r=R(s,a), s'\sim T(s'|s,a)$. For any task $M^*\in\mathcal{M}$ with reward and transition $R^*,T^*$, denote $x^*=(s,a,r^*,{s^*}')$ as a transition tuple generated in $M^*$ conditioned on $(s,a)$, where $r^*=R^*(s,a), {s^*}'\sim T^*(s'|s,a)$. Then we have 
\begin{equation*}
    I(z;M)-\log(N)\ge \mathbb{E}_{\mathcal{M}, x, z}{\left[\log(\frac{h(x,z)}{\sum_{M^*\in\mathcal{M}}{h(x^*,z)}})\right]}.
\end{equation*}
\end{theorem}

We give a proof of the theorem in Appendix \ref{appendix:thm}. Following InfoNCE, we approximate $h$ with the exponential of a score function $S(z^*,z)$, which is a similarity measure between the latent codes of two samples. We derive a sampling version of the tractable lower bound to be the transition encoder's learning objective
\begin{equation}
\label{eq:contrastive_loss}
    \max_{{\theta_1}}\sum_{\substack{M_i\in\mathcal{M} \\ x,x'\in X_i}}{\left[\log\left(\frac{\exp(S(z,z'))}{\sum_{M^*\in\mathcal{M}}{\exp(S(z,z^*))}}\right)\right]},
\end{equation}
where $\mathcal{M}$ is the set of training tasks, $x,x'$ are two transition tuples sampled from dataset $X_i$ and $z,z'$ are latent representations of $x,x'$. For tasks $M^*\in\mathcal{M}\setminus\{M_i\}$, $z^*$ is the representation of the transition tuple $x^*$ sampled in task $M^*$ conditioned on the same state-action pair in $x$. For $M^*=M_i$, we define $z^*=z'$ for consistent notation. We use cosine similarity for the score function in practice.

Following the literature, we name $(x,x')$ a positive pair and name $\{(x,x^*)\}_{M^*\in\mathcal{M}\setminus\{M\}}$  negative pairs. Eq. \eqref{eq:contrastive_loss} optimizes for a $N$-way classification loss to classify the positive pair out of all the pairs. To maximize the score of positive pairs, the transition encoder should extract shared features in the transitions of the same task. To decrease the score of negative pairs, the transition encoder should capture the essential variance of rewards and state transitions since state-action pairs are the same across tasks.

\subsection{Negative Pairs Generation}

Generating negative pairs in Eq. \eqref{eq:contrastive_loss} involves computing $(r,s')$ with tasks' reward functions and transition dynamics, which is impossible in the fully offline setting. \citet{mbml} introduce a reward relabeling method that fits a reward model for each task using offline datasets. In our setting, we also need to fit transition models with higher dimensions. However, the dataset for each task is usually small, making reward models and transition models overfit. When the overlap of state-action pairs between different tasks is small, relabeling with separately learned models can be inaccurate, making contrastive learning ineffective.

To generate high-quality negative samples, we design methods based on two major principles.

\textbf{Fidelity}: To ensure the task encoder learning on real transition tuples, the distribution of generated negative pairs should approximate the true distribution in the denominator of Eq. \eqref{eq:contrastive_loss}. Concretely, given $(s,a)$, the true distribution of $(r,s')$ in negative pairs is 
\begin{equation}
\label{eq:neg-distribution}
    p(r,s'|s,a)\propto\mathbb{E}_{M\sim P(M)}{\left[T(s'|s,a)\boldsymbol{1}\{R(s,a)=r\}\right]}
\end{equation}
\textbf{Diversity}: According to prior works \cite{generalized-data-transform, simclr}, the performance of contrastive learning depends on the diversity of negative pairs. More diversity increases the difficulty in optimizing InfoNCE and helps the encoder to learn meaningful representations. In image domains, large datasets, large negative batch sizes, and data augmentations are used to increase diversity. We adopt this principle in generating negative pairs of transitions.

We propose two approaches for negative pairs generation.

\textbf{Generative Modeling}: We observe that, though a state-action pair may not appear in all the datasets, it usually appears in a subset of training datasets. Fitting reward and transition models with larger datasets across tasks can reduce the prediction error. Thus, we propose to train a generative model over the union of training datasets $\cup_{i=1}^{N}{X_i}$, to approximate the offline data distribution $P_{\{X_i\}}(r,s'|s,a)$. When the distributions of $(s,a)$ are similar across tasks, this distribution approximates the distribution in Eq. \eqref{eq:neg-distribution}. 

We adopt conditional VAE (CVAE) \cite{cvae} for generative modeling. CVAE consists of a generator $p_\xi(r,s'|s,a,z)$ and a probabilistic encoder $q_\omega(z|s,a,r,s')$. The latent vector $z$ describes the uncertain factors for prediction, follows a prior Gaussian distribution $p(z)$. CVAE is trained to minimize the loss function 
\begin{equation}
\label{eq:cvae}
\begin{aligned}
    & \mathcal{L}_{\operatorname{CVAE}} =  -\mathbb{E}_{(s,a,r,s')\in \{X_i\}} \Big[ \mathbb{E}_{q_\omega(z|s,a,r,s')}\\
    & \ {[\log p_\xi(r,s'|s,a,z)]} - \operatorname{KL}[q_\omega(z|s,a,r,s')\|p(z)] \Big],
\end{aligned}
\end{equation}
where $\operatorname{KL}(\cdot \| \cdot)$ is KL-divergence, working as an information bottleneck to minimize the information in $z$.

\begin{algorithm}[t]
  \caption{Meta Training}
  \label{alg:train}
\begin{algorithmic}
  \STATE {\bfseries Input:} Datasets $\{X_i\}_{i=1}^N$; OMRL models $E_{\theta_1}$, $E_{\theta_2}$, $Q_\psi$, $\pi_\phi$
  \STATE {\bfseries A. If use generative modeling, pre-train CVAE:}
  \STATE Initialize CVAE $q_\omega, p_\xi$
  \REPEAT
  \STATE Update $\omega, \xi$ to minimize Eq. \eqref{eq:cvae}.
  \UNTIL{Done}
  \STATE {\bfseries B. Train the transition encoder:}
  \REPEAT
  \STATE Sample a task $M$ and two transition tuples $x,x'$
  \STATE $z=E_{\theta_1}(x),z'=E_{\theta_1}(x')$
  \FOR{$M^*\in\mathcal{M}$}
  \IF{use generative modeling}
  \STATE Sample $x^*$ from CVAE
  \ELSIF{use reward randomization}
  \STATE Add noise to the reward to get $x^*$
  \ENDIF
  \STATE $z^*=E_{\theta_1}(x^*)$
  \ENDFOR
  \STATE Compute Eq. \eqref{eq:contrastive_loss}
  \STATE Update $\theta_1$ to maximize Eq. \eqref{eq:contrastive_loss}
  \UNTIL{Done}
  \STATE {\bfseries C. Train the policy:}
  \REPEAT
  \STATE Sample a task dataset $X$ and a context $c$
  \STATE $z=E_{\theta_2}(E_{\theta_1}(c))$
  \STATE Augment the states in $X$ with z
  \STATE Update $\theta_2,\psi,\phi$ with offline RL algorithms on $X$
  \UNTIL{Done}
\end{algorithmic}
\end{algorithm}

\begin{algorithm}[t]
  \caption{Meta Test}
  \label{alg:test}
\begin{algorithmic}
  \STATE {\bfseries Input:} Trained models $E_{\theta_1}$, $E_{\theta_2}$, $Q_\psi$, $\pi_\phi$
  \STATE Sample a task $M$
  \STATE Collect a context trajectory $c$ with an arbitrary policy
  \STATE $z=E_{\theta_2}(E_{\theta_1}(c))$
  \REPEAT
  \STATE Observe $s$, execute $a\sim\pi_\phi(a|s,z)$, get $r$
  \UNTIL{Environment terminates}
\end{algorithmic}
\end{algorithm}


\begin{figure*}[!t]
	\centering
	\begin{minipage}[t]{0.33\textwidth}
		\centering
		\includegraphics[scale=0.128, trim={0cm, 0cm, 2cm, 0cm}]{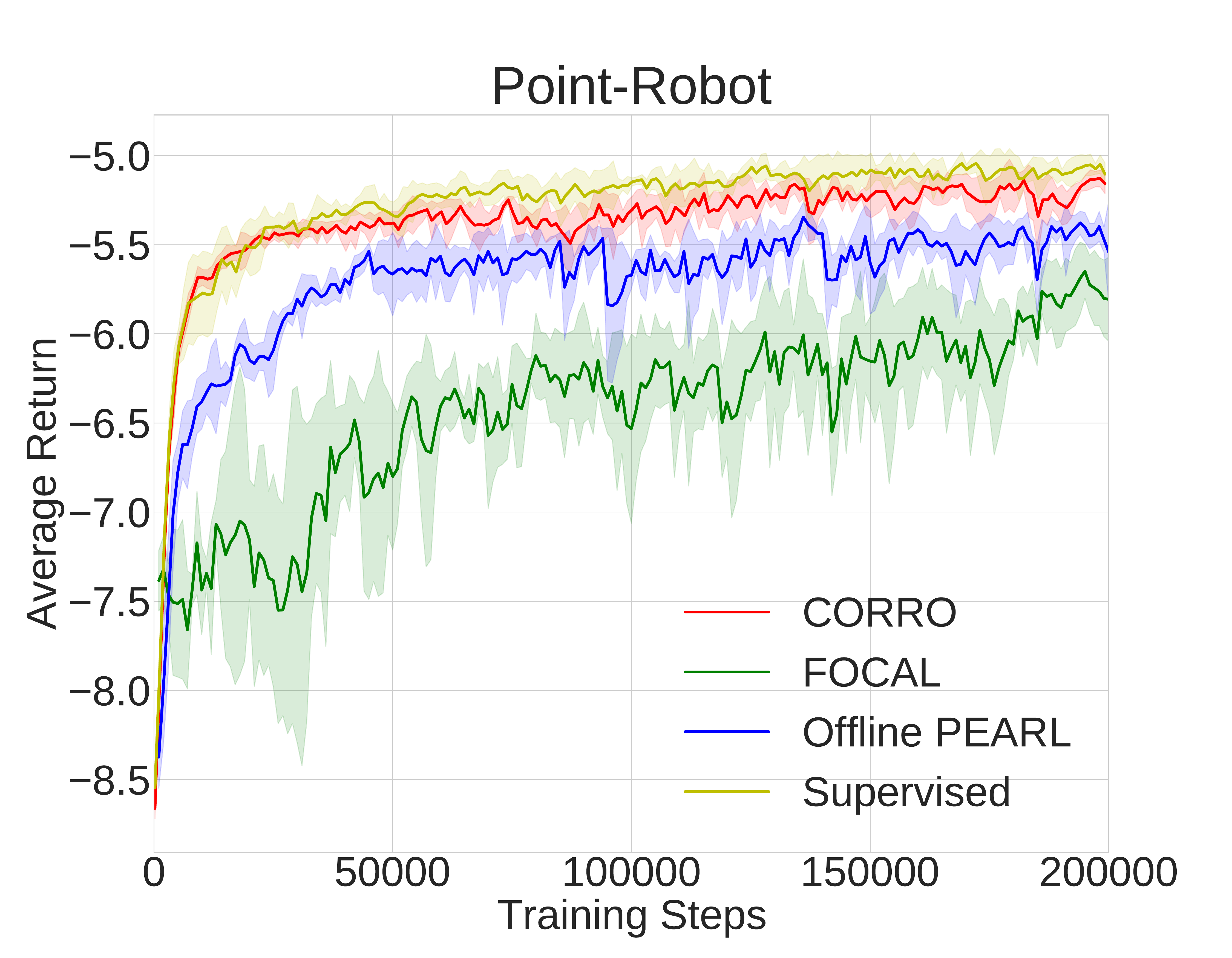}
	\end{minipage}
	\begin{minipage}[t]{0.33\textwidth}
		\centering
		\includegraphics[scale=0.128, trim={0cm, 0cm, 2cm, 0cm}]{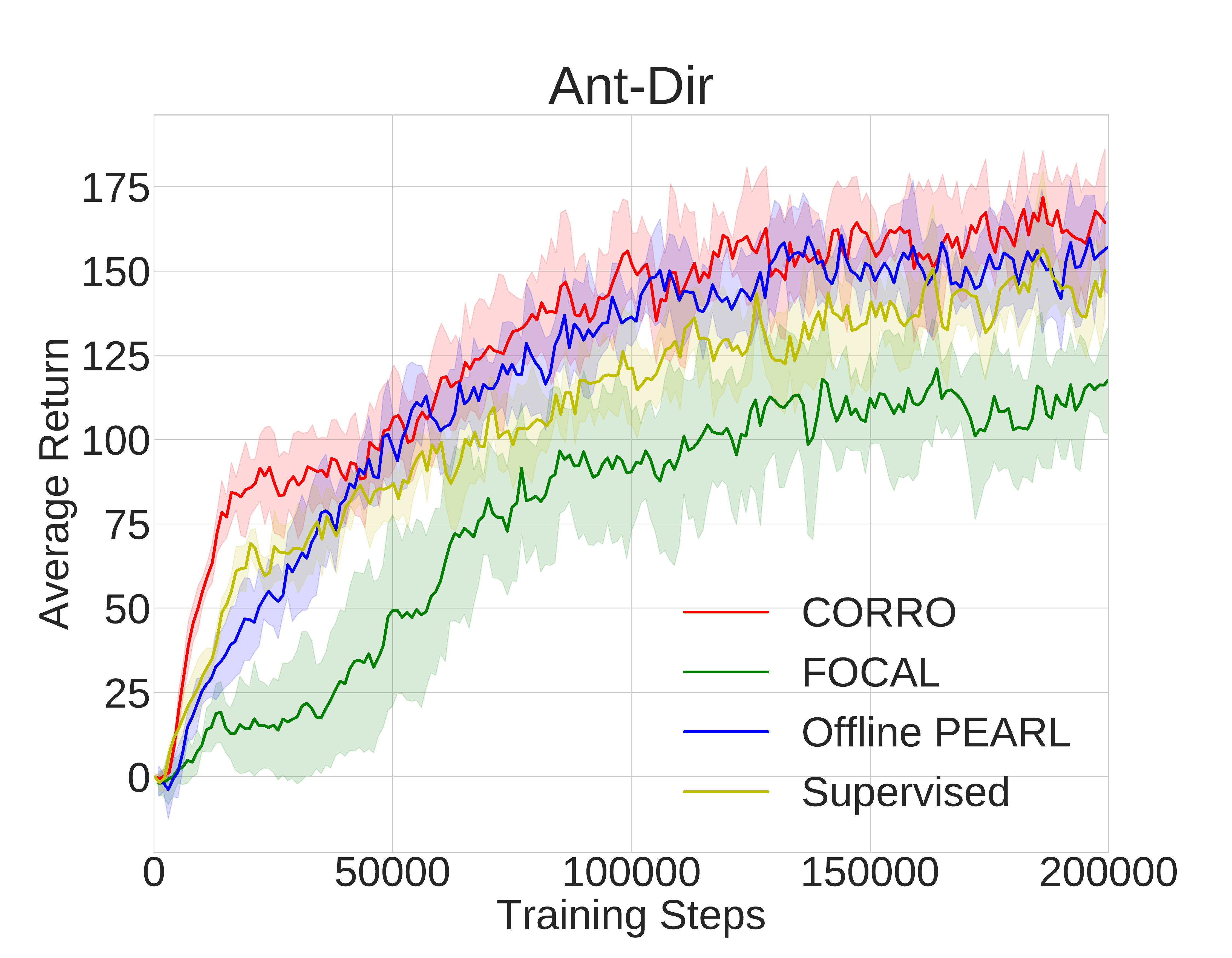}
	\end{minipage}
	\begin{minipage}[t]{0.33\textwidth}
		\centering
		\includegraphics[scale=0.128, trim={0cm, 0cm, 2cm, 0cm}]{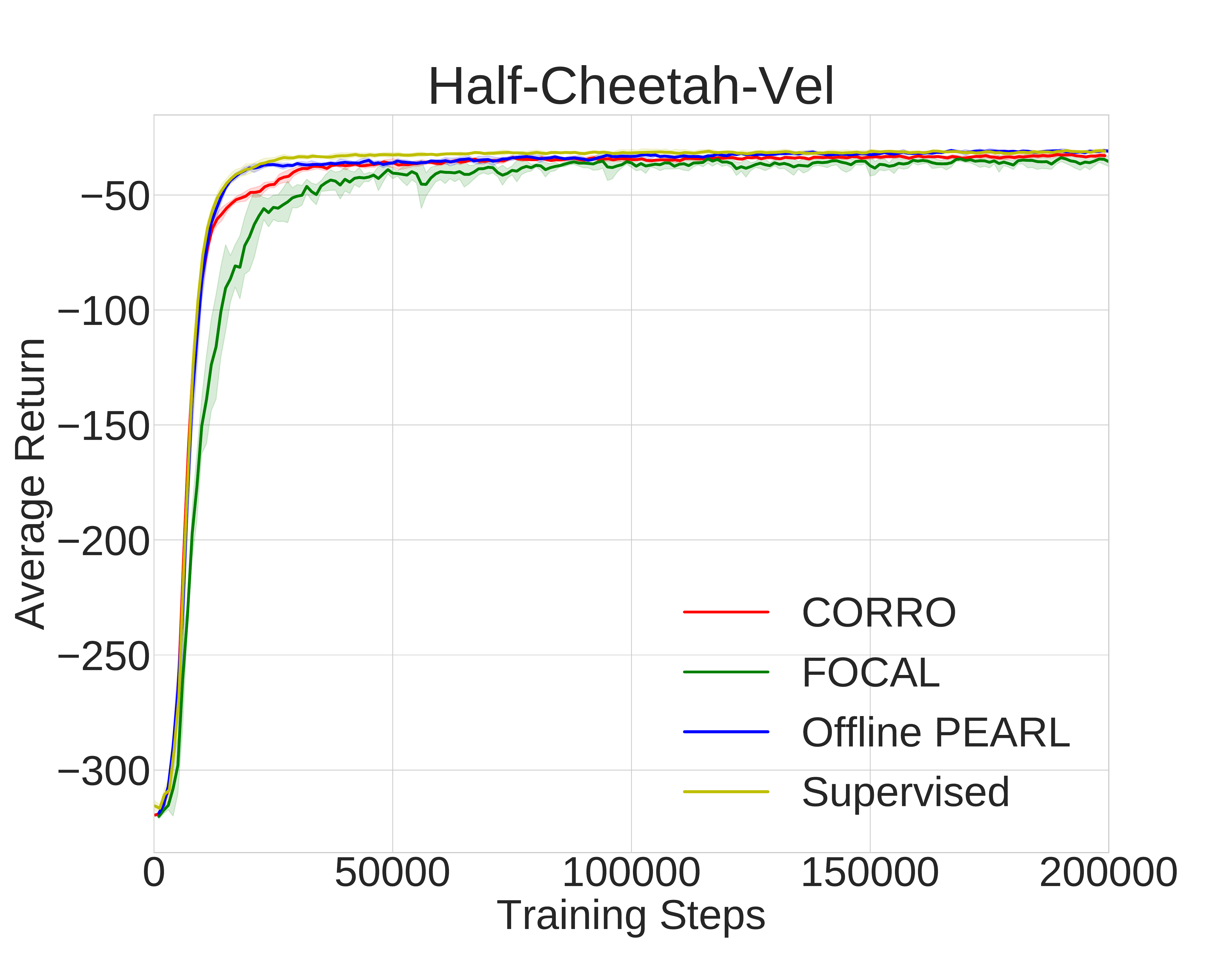}
	\end{minipage}
	\caption{Test returns of CORRO against the baselines in the environments with different reward functions: Point-Robot, Ant-Dir, and Half-Cheetah-Vel. The shaded region shows standard deviation across 5 seeds.
    }
    \label{fig:iid-reward}
\end{figure*}

\textbf{Reward Randomization}: In the cases where the overlap of state-action pairs between tasks is small, CVAE can collapse almost to a deterministic prediction model and the diversity in negative pairs decreases. When tasks only differ in reward functions, similar to data augmentation strategies in image domains, we can generate negative pairs via adding random noise to the reward. We have $r^*=r+\nu$, where the perturbation $\nu$ follows a noise distribution $p(\nu)$. Though reward randomization cannot approximate the true distribution of negative pairs, it provides an infinitely large space to generate diverse rewards, making contrastive learning robust.


\subsection{Algorithm Summary}
We summarize our training method in Algorithm \ref{alg:train}. 
The test process is shown in Algorithm \ref{alg:test}.
The code for our work is available at \href{https://github.com/PKU-AI-Edge/CORRO}{https://github.com/PKU-AI-Edge/CORRO}.

\section{Experiments}

In experiments, we aim to demonstrate: (1) The performance of CORRO on tasks adaptation in diverse task distributions and offline datasets; (2) The robustness of CORRO on task inference; and (3) How to choose the strategy of negative pairs generation.

\subsection{Experimental Settings}
We adopt a simple 2D environment and multi-task MuJoCo benchmarks to evaluate our method.

\textbf{Point-Robot} is a 2D navigation environment introduced in \citet{pearl}. Starting from the initial point, the agent should navigate to the goal location. Tasks differ in \textit{reward functions}, which describe the goal position. The goal positions are uniformly distributed in a square.

\textbf{Half-Cheetah-Vel, Ant-Dir} are multi-task MuJoCo benchmarks where tasks differ in \textit{reward functions}. In Half-Cheetah-Vel, the task is specified by the target velocity of the agent. The distribution of target velocity is $U[0,v_{\max}]$. In Ant-Dir, the task is specified by the goal direction of the agent's motion. The distribution of goal direction is $U[0,2\pi]$. These benchmarks are also used in \citet{macaw, focal, varibad}.

\textbf{Walker-Param, Hopper-Param} are multi-task MuJoCo benchmarks where tasks differ in \textit{transition dynamics}. For each task, the physical parameters of body mass, inertia, damping, and friction are randomized. The agent should adapt to the varying environment dynamics to accomplish the task. Previous works \cite{macaw, pearl} also adopt these benchmarks.

For each environment, 20 training tasks and 20 testing tasks are sampled from the task distribution. On each task, we use SAC \cite{sac} to train a single-task policy independently. The replay buffers are collected to be the offline datasets. During training, the context is a contiguous section of 200 transition tuples sampled from the offline dataset of the corresponding task. More details about experimental settings are available in Appendix~\ref{app:exp}.


\begin{figure}[!t]
\centering
\includegraphics[scale=0.105, trim={5cm, 1cm, 5.5cm, 3cm}, clip]{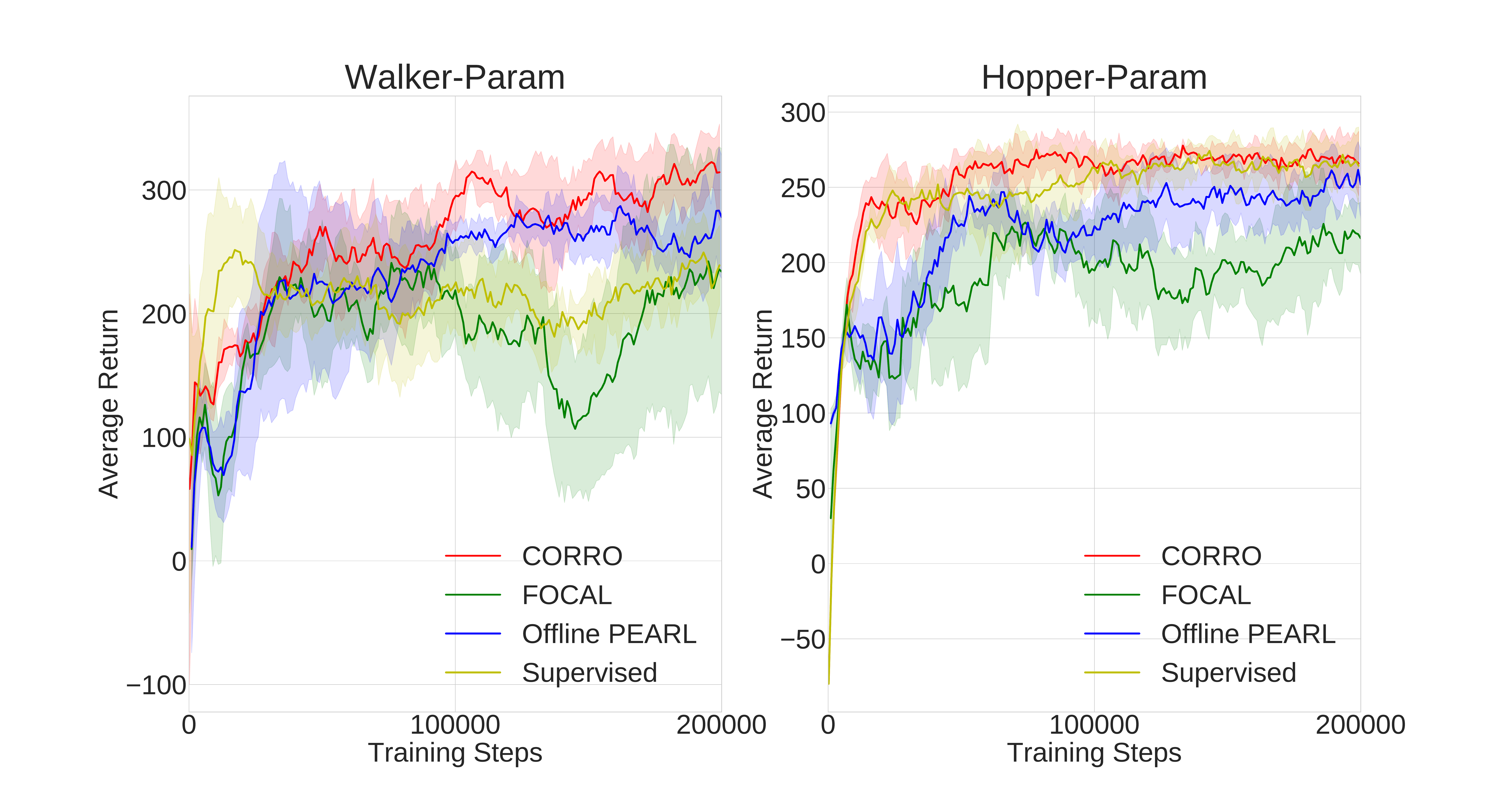}
\vspace{-0.1cm}
\caption{Test returns of CORRO against the baselines in the environments with different transition dynamics: Walker-Param and Hopper-Param. The shaded region shows standard deviation across 5 seeds.}
\label{fig:iid-transition}
\end{figure}

\newcolumntype{L}[1]{>{\raggedright\arraybackslash}p{#1}}
\newcolumntype{C}[1]{>{\centering\arraybackslash}p{#1}}
\newcolumntype{R}[1]{>{\raggedleft\arraybackslash}p{#1}}
\newcommand{\tabincell}[2]{\begin{tabular}{@{}#1@{}}#2\end{tabular}}

\renewcommand\tabcolsep{5pt}
\begin{table*}[!t]
\renewcommand\arraystretch{1.2}
\caption{Average test returns of CORRO against baselines with different types of context exploration policy. \textit{IID} means the context distribution is the same to the training dataset, while \textit{OOD} means the context is out-of-distribution. }
\label{table:ood_test}
\vspace{0.2cm}
\centering
\begin{small}
\begin{tabular}{@{}ccccccccc@{}}
\toprule
\multicolumn{1}{c}{\multirow{2}{*}{\textbf{Environment}}} & \multicolumn{2}{c}{\textbf{Supervised.}} & \multicolumn{2}{c}{\textbf{Offline PEARL}} & \multicolumn{2}{c}{\textbf{FOCAL}} & \multicolumn{2}{c}{\textbf{CORRO}}  \\ \cmidrule(lr){2-9} 
\multicolumn{1}{c}{}  & \textit{IID} & \textit{OOD} &  \textit{IID} & \textit{OOD} &  \textit{IID} & \textit{OOD} &  \textit{IID} & \textit{OOD} \\ \midrule\midrule

\textbf{\tabincell{c}{{Point-Robot}}}
& \textbf{-4.89}\scriptsize{$\pm 0.10$} & \textbf{-5.84}\scriptsize{$\pm 0.14$}
& -5.4\scriptsize{$\pm 0.17$} & -6.74\scriptsize{$\pm 0.19$} 
& -6.06\scriptsize{$\pm 0.42$} & -7.34\scriptsize{$\pm 0.20$} 
& \textbf{-5.19}\scriptsize{$\pm 0.05$} & \textbf{-6.39}\scriptsize{$\pm 0.05$}  \\

\textbf{\tabincell{c}{{Ant-Dir}}} 
& {136}\scriptsize{$\pm 17.6$} & {131.7}\scriptsize{$\pm 11.4$} 
& 155.4\scriptsize{$\pm 24.4$} & {141.5}\scriptsize{$\pm 11.3$} 
& {109.8}\scriptsize{$\pm 12.8$} & {53.5}\scriptsize{$\pm 16.4$} 
& \textbf{156.8}\scriptsize{$\pm 35.2$} & \textbf{154.7}\scriptsize{$\pm 25.8$}  \\ 

\textbf{\tabincell{c}{{Half-Cheetah-Vel}}} 
& \textbf{-31.6}\scriptsize{$\pm 0.7$} & \textbf{-32.1}\scriptsize{$\pm 0.9$} 
& \textbf{-31.2}\scriptsize{$\pm 0.5$} & -242.7\scriptsize{$\pm 6.0$} 
& -38.0\scriptsize{$\pm 4.0$} & -204.1\scriptsize{$\pm 9.5$} 
& -33.7\scriptsize{$\pm 1.1$} & \textbf{-89.7}\scriptsize{$\pm 7.4$} \\ 
\midrule

\textbf{\tabincell{c}{{Walker-Param}}} 
& 232.7\scriptsize{$\pm 29.2$} & {221.2}\scriptsize{$\pm 43.4$} 
& 259.1\scriptsize{$\pm 48.2$} & 254.7\scriptsize{$\pm 35.8$} 
& 225.4\scriptsize{$\pm 56.4$} & 193.3\scriptsize{$\pm 151.5$}  
& \textbf{301.5}\scriptsize{$\pm 37.9$} & \textbf{284.0}\scriptsize{$\pm 19.3$}  
\\ 

\textbf{\tabincell{c}{{Hopper-Param}}} 
& \textbf{269.2}\scriptsize{$\pm 20.3$} & 251.9\scriptsize{$\pm 28.8$} 
& 244.0\scriptsize{$\pm 18.5$} & 236.6\scriptsize{$\pm 18.5 $} 
& 195.6\scriptsize{$\pm 62.3$} & 199.7\scriptsize{$\pm 51.9$}  
& \textbf{267.6}\scriptsize{$\pm 25.6$} & \textbf{268.0}\scriptsize{$\pm 13.8$}
\\ \bottomrule 

\end{tabular}
\end{small}
\end{table*}

To evaluate the performance of our method, we introduce several representative context-based fully-offline methods to compare with.

\textbf{Offline PEARL} is a direct offline extension of PEARL \cite{pearl}. The context encoder is a permutation invariant MLP, updated together with policy and Q-function using offline RL algorithms. No additional terms are used for task representation learning. This baseline is introduced in \cite{mbml, mbomrl, focal}.

\textbf{FOCAL} \cite{focal} uses metric learning to train the context encoder. Positive and negative pairs are trajectories sampled from the same and the different tasks, respectively. 
Note that while the original FOCAL uses an offline RL method BRAC \cite{behavior-regularize}, we implement FOCAL with SAC to better study the task representation learning problem.

\textbf{Supervised Task Learning} assumes that \textit{the ground truth task descriptions} can be accessed during training. Based on PEARL, we add $L_2$ loss to supervise the context encoder. We introduce this supervised learning method as a strong baseline. Note that the ground truth task labels are usually \textit{unavailable} in meta-RL settings.

For fairness, we replace the policy learning methods of baselines and CORRO with SAC. The hyperparameters for policy learning are fixed across different methods. All the experiments are conducted over 5 different random seeds. All hyperparameters are available in Appendix~\ref{app:hyper}.

\subsection{Tasks Adaptation Performance}
\label{sec:iid_test}

To evaluate the performance on tasks adaptation, we use tasks in the test set, sample one trajectory in the pretrained replay buffer of each task to be the context, then test the agent in the environment conditioned on the context. The performance is measured by the average return over all the test tasks. In this experiment, there is \textit{no distributional shift of contexts} between training and test, since they are generated by the single-task policies learned with the same algorithm.

Figure \ref{fig:iid-reward} shows the test results in the environments with different reward functions. In Point-Robot, CORRO and the supervised baseline outperform others. In Ant-Dir, CORRO outperforms all the baselines including the supervised learning method. Although the value of return in Point-Robot varies a little, we have to mention that the return in this environment is insensitive to the task performance: the mean return of an optimal policy is around -5, while the mean return of a random policy is around -9. Benefited from the pretrained task encoder, CORRO shows great learning efficiency, and reaches high test returns after 20k steps of offline training. Though in Half-Cheetah-Vel we cannot distinguish the performance of different methods, we will demonstrate their difference when context distribution shifts exist in Section \ref{sec:ood_test} and the difference on the learned representations in Section \ref{sec:vis}.

Figure \ref{fig:iid-transition} demonstrates the results in the environments with varying transition dynamics. In both Walker-Param and Hopper-Param, CORRO outperforms the baselines on task adaptation returns. Further, CORRO is more stable on policy improvement, while the performance of FOCAL sometimes degenerates during the offline training.

\subsection{Robust Task Inference}
\label{sec:ood_test}

A robust task encoder should capture a compact representation of rewards and transition dynamics from the context, eliminating the information of the behavior policy. Conditioned on the context collected by an arbitrary behavior policy, the agent performs accurate task adaptation and reaches high returns. To measure the robustness, we design experiments of out-of-distribution (\textit{OOD}) test. Denote the saved checkpoints of data collection policies for all the tasks during different training periods as $\{\pi_{\beta}^{i,t}\}_{i,t}$, where $i$ is the task and $t$ is the training epoch. For each test task $M_j^{\text{test}}$, we sample $\pi_{\beta}^{i,t}$ as the behavior policy to collect a context rollout, and test the agent conditioned on the context. Because $\pi_{\beta}^{i,t}$ is a policy trained on an arbitrary task, the context distribution is never seen during meta training. Then, we sample $S$ behavior policies, measure the \textit{OOD} test performance as $\frac{1}{NS}\sum_{k=1}^{S}\sum_{j=1}^N{J(M_j^{\text{test}}, \pi_{\beta}^{i_k,t_k})}$, where $J(M,\pi)$ means the average return in task $M$ with context collection policy $\pi$. Accordingly, we denote the standard in-distribution test in Section \ref{sec:iid_test} as \textit{IID} test.

Table \ref{table:ood_test} summarizes the \textit{IID} and \textit{OOD} test results of different methods. All the test models are from the last training epoch. While the performance of all the methods degenerates from \textit{IID} to \textit{OOD}, CORRO outperforms Offline PEARL and FOCAL in \textit{OOD} test by a large margin. In Half-Cheetah-Vel, Offline PEARL and FOCAL almost fail in task inference with distributional shift. The slight performance degeneration shows the robustness of CORRO. Note that although the supervised task learning method does not degenerate in Half-Cheetah-Vel, it uses the true task labels to learn the relation between contexts and tasks, which is not allowed in common meta-RL settings.

\subsection{Latent Space Visualization}
\label{sec:vis}

\begin{figure}[!t]
    \begin{center}
        \includegraphics[scale=0.5, trim={0cm, 3.8cm, 19cm, 0cm}, clip]{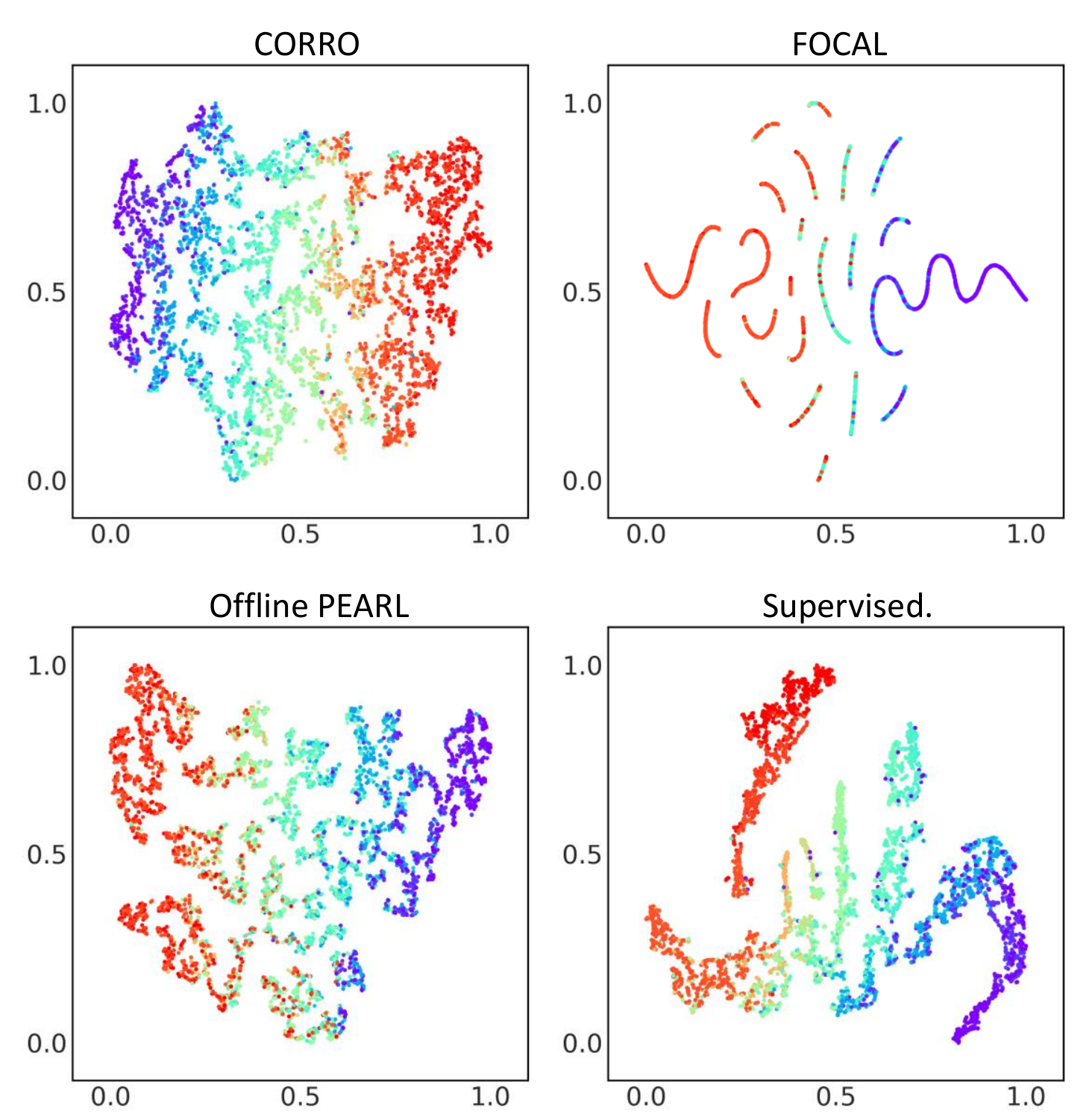}
    \end{center}
    \vspace{-0.2cm}
    \caption{The 2D projection of the learned task representation space in Half-Cheetah-Vel. Points are uniformly sampled from test tasks. Tasks of velocities from 0 to 3 are mapped to rainbow colors, from red to purple.
    }
    \label{fig:embedding}
\end{figure}

To qualitatively analyze the learned task representation space, we visualize the task representations by projecting the embedding vectors into 2D space via t-SNE \cite{t-SNE}. For each test task, we sample 200 transition tuples in the replay buffer to visualize. As shown in Figure \ref{fig:embedding}, compared to the baselines, CORRO distinguishes transitions from different tasks better. CORRO represents the task space of Half-Cheetah-Vel into a 1D manifold, where the distance between latent vectors is correlated to the difference between their target velocities, as shown in the color of the points.

\subsection{Ablation: Negative Pairs Generation}

\begin{table}[!t]
\renewcommand\arraystretch{1.2}
\begin{center}
\caption{Ablation study on negative pairs generation methods in Half-Cheetah-Vel environment (the upper four rows) and Point-Robot environment (the lower four rows).}
\label{table:ablation}
\vspace{0.2cm}
\small
\begin{tabular}{@{}cccc@{}}
\toprule
\textbf{Method} & \textbf{Contrastive Loss} & \textbf{\textit{IID} Return} & \textbf{\textit{OOD} Return}  \\ \midrule\midrule

\textbf{\tabincell{c}{{Generative}}} & 0.07 &  \textbf{-33.7}\scriptsize{$\pm 1.1$} & {-89.7}\scriptsize{$\pm 7.4 $}   \\ 

\textbf{\tabincell{c}{{Randomize}}} & 0.83 &  -34.3\scriptsize{$\pm 1.5$} & \textbf{-84.5}\scriptsize{$\pm 1.3 $}   \\ 

\textbf{\tabincell{c}{{Relabeling}}} & 0.04 & -40.8\scriptsize{$\pm 1.5$} & -245.3\scriptsize{$\pm 12.9$} \\ 

\textbf{\tabincell{c}{{None}}} & 1.20 &   -34.1\scriptsize{$\pm 2.4$} & {-97.6}\scriptsize{$\pm 3.1 $}   \\ \midrule\midrule

\textbf{\tabincell{c}{{Generative}}} & 2.83 &  -9.41\scriptsize{$\pm 0.42$} & -9.42\scriptsize{$\pm 0.42 $}   \\ 

\textbf{\tabincell{c}{{Randomize}}} & {0.54} &  \textbf{-5.19}\scriptsize{$\pm 0.05$} & \textbf{-6.39}\scriptsize{$\pm  0.05$}   \\ 

\textbf{\tabincell{c}{{Relabeling}}} & 0.04 & -9.22\scriptsize{$\pm 0.24$} & -9.27\scriptsize{$\pm 0.22$} \\ 

\textbf{\tabincell{c}{{None}}} & 1.46 &   -5.24\scriptsize{$\pm 0.27$} & {-6.52}\scriptsize{$\pm 0.08$}   \\ \bottomrule

\end{tabular}
\end{center}
\end{table}

\renewcommand\tabcolsep{16pt}
\begin{table*}[!t]
\renewcommand\arraystretch{1.2}
\caption{Average test returns of CORRO against baselines with a random exploration policy for context collection. }
\label{table:random_test}
\vspace{0.2cm}
\centering
\begin{small}
\begin{tabular}{ccccc}
\toprule
\textbf{Environment} & \textbf{Supervised.} & \textbf{Offline PEARL} & \textbf{FOCAL} & \textbf{CORRO}  \\ \midrule\midrule

\textbf{\tabincell{c}{{Point-Robot}}}
& -5.32\scriptsize{$\pm 0.20$} & -7.06\scriptsize{$\pm 0.99$}
& -8.64\scriptsize{$\pm 0.26$} & \textbf{-5.59}\scriptsize{$\pm 0.57$}  \\

\textbf{\tabincell{c}{{Ant-Dir}}} 
& 149.8\scriptsize{$\pm 20.5$} & {148.4}\scriptsize{$\pm 35.3$} 
& 89.8\scriptsize{$\pm 8.7$} & \textbf{163.0}\scriptsize{$\pm 35.8$}  \\ 

\textbf{\tabincell{c}{{Half-Cheetah-Vel}}} 
& -37.6\scriptsize{$\pm 0.8$} & \textbf{-35.4}\scriptsize{$\pm 1.8$} 
& {-41.6}\scriptsize{$\pm 3.3$} & -42.9\scriptsize{$\pm 0.7$}  \\ 
\midrule

\textbf{\tabincell{c}{{Walker-Param}}} 
& 221.7\scriptsize{$\pm 91.1$} & {276.6}\scriptsize{$\pm 37.7$} 
& 245.6\scriptsize{$\pm 67.8$} & \textbf{300.5}\scriptsize{$\pm 34.2$} \\ 

\textbf{\tabincell{c}{{Hopper-Param}}} 
& 253.5\scriptsize{$\pm 21.2$} & 245.9\scriptsize{$\pm 18.9$} 
& 203.6\scriptsize{$\pm 46.6$} & \textbf{273.3}\scriptsize{$\pm 3.9 $} 
\\ \bottomrule 

\end{tabular}
\end{small}
\end{table*}

Negative pairs generation is a key component of CORRO. Generating more diverse negative samples can make contrastive learning robust and efficient, while the distribution of negative pairs should also be close to the true distribution as described in Section \ref{sec:contrastive}. Along with our design of generative modeling and reward randomization, we also introduce \textbf{Relabeling} and \textbf{None} method. \textbf{Relabeling} is to separately learn a reward model and a transition model on each offline training dataset. The negative pair is generated by first sampling a model, then relabeling $(r,s')$ in the transition tuple. \textbf{None} is a straightforward method where we simply construct negative pairs with transition tuples from different tasks and $(s,a)$ in the transitions can be different. 

As shown in Table \ref{table:ablation}, in Half-Cheetah-Vel, our methods using generative modeling and reward randomization outperform None and Relabeling in \textit{OOD} test, due to better approximating the distribution of negative pairs. 

In Half-Cheetah-Vel, we notice that the None method achieves the average return of -97.6 in \textit{OOD} test and outperforms FOCAL, of which the average return is -204.1 according to Table \ref{table:ood_test}. 
They have two main differences: 1. samples for representation learning are
transition tuples in ‘None’, but are trajectories in FOCAL; 2. the contrastive loss. To find which affects the OOD generalization more, we re-implement FOCAL with our InfoNCE loss. It achieves an average return of -41.9 in the IID test and -208.2 in the OOD test, which is close to FOCAL’s performance. 
Thus, we argue that focusing on transition tuples rather than whole trajectories contributes to OOD generalization.
Our approach naturally discourages the task encoder to extract the feature of behavior policy from the trajectory sequence.

While generative modeling and reward randomization has similar performance in Half-Cheetah-Vel, the latter solves Point-Robot better. Because in Point-Robot, behavior policies in different tasks move in different directions, making state-action distributions non-overlapping. CVAE collapses into predicting the transition of the specific task where the state-action pair appears, and cannot provide diverse negative pairs to support contrastive learning. 

The column of contrastive loss shows that the two proposed methods generate diverse and effective negative pairs in the right circumstances, improving the learning of task encoder. Though the relabeling method has the best sample diversity, it underperforms other methods in task adaptation due to the inaccurate predictions on unseen data.

\subsection{Adaptation with an Exploration Policy}

Context exploration is a crucial part of meta-RL. Existing works either explore with the meta policy before adaptation \cite{varibad, maml}, or learn a separate exploration policy \cite{explore-then-execute, explore-context-meta-rl}. To test the task adaptation performance of CORRO with an exploration policy, we adopt a uniformly random policy to collect the context. As shown in Table \ref{table:random_test}, CORRO and the supervised baseline outperform others in Point-Robot. In Half-Cheetah-Vel, all the methods have great adaptation performance compared with the results in Table \ref{table:ood_test}. In the other three environments, CORRO outperforms all the baselines.

\section{Conclusion}
In this paper, we address the task representation learning problem in the fully offline setting of meta-RL. We propose a contrastive learning framework over transition tuples and several methods for negative pairs generation. In experiments, we demonstrate how our method outperforms prior methods in diverse environments and task distributions. We introduce \textit{OOD} test to quantify the superior robustness of our method and ablate the choice of negative pairs generation. Though, our study has not considered to learn an exploration policy for context collection, and to perform moderate environmental interactions when training data is extremely limited. We leave these directions as future work.

\section*{Acknowledgements}
We would like to thank the anonymous reviewers for their useful comments to improve our work. This work is supported in part by NSF China under grant 61872009.

\bibliography{ref2}
\bibliographystyle{icml2022}

\newpage
\appendix
\onecolumn


\section{Contrastive Learning Objective}
\label{appendix:thm} 

In this section, we first introduce a lemma. Then we give a proof of the main theorem in the text.

\begin{lemma}
\label{lem:1}
Given $M\sim P(M)$ with reward $R$ and transition $T$, $(s,a)$ follows some distribution, $x=(s,a,r,s'), r=R(s,a), s'\sim T(s'|s,a), z\sim P(z|x)$, then we have
\begin{equation*}
    \frac{P(M|z)}{P(M)} = \mathbb{E}_{x}{\left[\frac{P(z|x)}{P(z)}\right]}.
\end{equation*}
\end{lemma}
\begin{proof} 
\begin{equation*}
    \begin{aligned}
    \frac{P(M|z)}{P(M)} &= \frac{P(z|M)}{P(z)} \\
    &= \int_x{ \frac{P(x|M)P(z|x)}{P(z)} dx} \\
    &= \mathbb{E}_{x}{\left[\frac{P(z|x)}{P(z)}\right]}
    \end{aligned}
\end{equation*}
\end{proof}


\begin{theorem}
\label{thm:main}
Let $\mathcal{M}$ be a set of tasks following the task distribution, $|\mathcal{M}|=N$. $M\in\mathcal{M}$ is the first task with reward and transition $R,T$. Let $x=(s,a,r,s'), z\sim P(z|x), h(x,z)=\frac{P(z|x)}{P(z)}$, where $(s,a)$ follows an arbitrary distribution, $r=R(s,a), s'\sim T(s'|s,a)$. For any task $M^*\in\mathcal{M}$ with reward and transition $R^*,T^*$, denote $x^*=(s,a,r^*,{s^*}')$ as a transition tuple generated in $M^*$ conditioned on $(s,a)$, where $r^*=R^*(s,a), {s^*}'\sim T^*(s'|s,a)$. Then we have 
\begin{equation*}
    I(z;M)-\log(N)\ge \mathbb{E}_{\mathcal{M}, x, z}{\left[\log(\frac{h(x,z)}{\sum_{M^*\in\mathcal{M}}{h(x^*,z)}})\right]}.
\end{equation*}
\end{theorem}

\begin{proof} 
Using Lemma \ref{lem:1} and Jensen's inequality, we have
\begin{equation*}
    \begin{aligned}
    I(z;M)-\log(N) &= \mathbb{E}_{{z,\mathcal{M}}}{\log{\left[ \frac{p(M|z)}{p(M)} \right]}} - \log{N} \\
     &= \mathbb{E}_{{z,\mathcal{M}}}{\left[ \log{\mathbb{E}_{x}{\left[ \frac{p(z|x)}{p(z)} \right]}} \right]} - \log{N} \\
     &\ge \mathbb{E}_{{z,\mathcal{M}}} \mathbb{E}_{x}{\left[ \log\left(\frac{1}{\frac{p(z)}{p(z|x)}N}\right) \right]} \\
     &\ge \mathbb{E}_{z,{\mathcal{M}}} \mathbb{E}_{x}{\left[ \log\left(\frac{1}{1+\frac{p(z)}{p(z|x)}(N-1)}\right) \right]} \\
     &= \mathbb{E}_{z,{\mathcal{M}}} \mathbb{E}_{x}{\left[ -\log\left( 1+\frac{p(z)}{p(z|x)}(N-1) \right) \right]} \\
     &= \mathbb{E}_{{z,\mathcal{M}}} \mathbb{E}_{x}{\left[ -\log\left( 1+\frac{p(z)}{p(z|x)}(N-1) \mathbb{E}_{M^*\in {\mathcal{M}}\setminus\{M\}}{\left[\frac{p(z|x^*)}{p(z)}\right]} \right) \right]} \\
     &= \mathbb{E}_{{z,\mathcal{M}}} \mathbb{E}_{x}{\left[ -\log\left( 1+\frac{p(z)}{p(z|x)}\sum_{M^*\in {\mathcal{M}}\setminus\{M\}}{\frac{p(z|x^*)}{p(z)}} \right) \right]} \\
     &= \mathbb{E}_{\mathcal{M}, x, z}{\left[ \log\left( \frac{\frac{p(z|x)}{p(z)}}{\frac{p(z|x)}{p(z)} + \sum_{M^*\in {\mathcal{M}}\setminus\{M\}}{\frac{p(z|x^*)}{p(z)}}} \right) \right]} \\
     &= \mathbb{E}_{\mathcal{M}, x, z}{\left[ \log(\frac{h(x,z)}{\sum_{M^*\in {\mathcal{M}}}{h(x^*,z)}}) \right]}
    \end{aligned}
\end{equation*}
\end{proof}

\section{Environment Details}
\label{app:exp}
In this section, we present the details of the environments in our experiments.

\textbf{Point-Robot:} The start position is fixed at $(0,0)$ and the goal location $g$ is sampled from $U[-1,1]\times U[-1,1]$. The reward function is defined as $r_t=-\|s_t-g\|_2$, where $s_t$ is the current position. The maximal episode steps is set to 20.

\textbf{Ant-Dir:} The goal direction is sampled from $\theta\sim U[0,2\pi]$. The reward function is $r_t=v_x\cos{\theta}+v_y\sin{\theta}$, where $(v_x,v_y)$ is the horizontal velocity of the ant. The maximal episode steps is set to 200.

\textbf{Half-Cheetah-Vel:} The goal velocity is sampled from $v_g\sim U[0,3]$. The reward function is $r_t=-|v_t-v_g|-\frac{1}{2}\|a_t\|^2_2$, where $v_t$ is the current forward velocity of the agent and $a_t$ is the action. The maximal episode steps is set to 200.

\textbf{Walker-Param:} Transition dynamics are varied in body mass and frictions, described by 32 parameters. Each parameter is sampled by multiplying the default value with $1.5^\mu, \mu\sim U[-3,3]$. The reward function is $r_t=v_t-10^{-3}\cdot\|a_t\|^2_2+1$, where $v_t$ is the current forward velocity of the walker and $a_t$ is the action. The episode terminates when the height of the walker is less than 0.5. The maximal episode steps is also set to 200.

\textbf{Hopper-Param:} Transition dynamics are varied in body mass, inertia, damping and frictions, described by 41 parameters. Each parameter is sampled by multiplying the default value with $1.5^\mu, \mu\sim U[-3,3]$. The reward function is $r_t=v_t-10^{-3}\cdot\|a_t\|^2_2$, where $v_t$ is the current forward velocity of the hopper and $a_t$ is the action. The maximal episode steps is set to 200.

\section{Experimental Details}
\label{app:hyper}

In Table \ref{table:hyperparams1} and \ref{table:hyperparams2}, we list the important configurations and hyperparameters in the data collection and meta training phases that we used to produce the experimental results. 

\begin{table*}[h]
\renewcommand\arraystretch{1.2}
\caption{Configurations and hyperparameters used in dataset collection to produce all the experimental results. }
\label{table:hyperparams1}
\vspace{0.2cm}
\centering
\begin{small}
\begin{tabular}{@{}c|c|c|c|c|c@{}}
\toprule

Configutations & Point-Robot & Ant-Dir & Half-Cheetah-Vel & Walker-Param &  Hopper-Param  \\ \midrule\midrule

Dataset size & 2100 & 2e4 & 2e5 & 2e4 & 6e4 \\
Training steps & 1e3 & 4e4 & 1e5 & 6e5 & 2e5 \\
Batch size & 256 & 256 & 256 & 256 & 256 \\
Network width & 32 & 128 & 128 & 128 & 128 \\
Network depth & 3 & 3 & 3 & 3 & 3 \\
Learning rate & 3e-4 & 3e-4 & 3e-4 & 3e-4 & 3e-4 \\
\bottomrule 
\end{tabular}
\end{small}
\end{table*}

\begin{table*}[h]
\renewcommand\arraystretch{1.2}
\caption{Configurations and hyperparameters used in offline meta training to produce all the experimental results. }
\label{table:hyperparams2}
\vspace{0.2cm}
\centering
\begin{small}
\begin{tabular}{@{}c|c|c|c|c|c@{}}
\toprule

Configutations & Point-Robot & Ant-Dir & Half-Cheetah-Vel & Walker-Param &  Hopper-Param  \\ \midrule\midrule

Negative pairs & Randomize & Randomize & Generative &  Generative &  None \\
$p(\nu)$ & $\mathcal{N}(0,0.5)$ & $\mathcal{N}(0,0.5)$ & -- & -- & -- \\
Latent space dim & 5 & 5 & 5 & 32 & 40 \\
Task batch size & 16 & 16 & 16 & 16 & 16 \\
Training steps & 2e5 & 2e5 & 2e5 & 2e5 & 2e5 \\
RL batch size & 256 & 256 & 256 & 256 & 256 \\
Contrastive batch size & 64 & 64 & 64 & 64 & 64 \\
Negative pairs number & 16 & 16 & 16 & 16 & 16 \\
RL network width & 64 & 256 & 256 & 256 & 256 \\
RL network depth & 3 & 3 & 3 & 3 & 3 \\
Encoder width & 64 & 64 & 64 & 128 & 128 \\
Learning rate & 3e-4 & 3e-4 & 3e-4 & 3e-4 & 3e-4 \\

\bottomrule 
\end{tabular}
\end{small}
\end{table*}

\end{document}